%% file: maxBayes9.tex
\documentclass[onecolumn, final]{IEEEtran}
\usepackage[active]{srcltx}
\input{mathheader}

\DeclareMathOperator{\argmax}{arg max}
\DeclareMathOperator{\rank}{rank}

\usepackage{cite}

\title{Bounds on the Bayes Error Given Moments}

\author{%
Bela A. Frigyik%
\thanks{Frigyik and Gupta are with the Department of Electrical Engineering,
  University of Washington, Seattle, WA 98195  (e-mail:
  frigyik@uw.edu, gupta@ee.washington.edu).  This work was supported by the
  United States PECASE Award.}  and
Maya R. Gupta%
\thanks{
}}

\begin{document}

\maketitle

\begin{abstract} % try to keep to 50 words
We show how to compute lower bounds for the supremum Bayes error if the class-conditional distributions must satisfy moment constraints, where the supremum is with respect to the unknown class-conditional distributions. Our approach makes use of Curto and Fialkow's solutions for the truncated moment problem. The lower bound shows that the popular Gaussian assumption is not robust in this regard. We also construct an upper bound for the supremum Bayes error by constraining the decision boundary to be linear.
\end{abstract}


% possible reviewers: Venu Veeravalli
% all the people in the references

\section{Introduction}

A standard approach in pattern recognition is to
estimate the first two moments of each class-conditional distribution
from training samples, and then assume the unknown distributions are
Gaussians. Depending on the exact assumptions, this approach is called
\emph{linear} or \emph{quadratic discriminant analysis} (QDA)
\cite{HTF,BDA}. Gaussians are known to maximize entropy given the
first two moments \cite{CoverThomas} and to have other nice mathematical properties, but
how robust are they with respect to maximizing the Bayes error? To answer that, in this paper we investigate the more general question: ``What is the maximum possible
Bayes error given moment constraints on the class-conditional
distributions?"

We present both a lower bound and an upper bound for the maximum possible Bayes error.  The lower bound means that there exists a set of class-conditional distributions that have the given moments and have a Bayes error above the given lower bound.  The upper bound means that no set of class-conditional distributions can exist that have the given moments and have a higher Bayes error than the given upper bound.

Our results provide some insight into how confident one can be in a classifier
if one is confident in the estimation of the first $n$ moments. In particular,
given only the certainty that two equally-likely classes have different means (and no trustworthy estimate
of their variances), we show that the Bayes error could be $1/2$, that is, the classes may not be separable at all. Given the first two moments, our results show that the popular Gaussian assumption for the class distributions is fairly optimistic  - the true Bayes error could be much worse. However, we show that the closer the class variances are,  the more robust the Gaussian assumption is.  In general, the given lower-bound may be a helpful way to assess the robustness of the assumed distributions used in generative classifiers.

The given upper bound may also be useful in practice. Recall that the Bayes error is the error that would arise from the optimal decision boundary. Thus, if one has a classifier and finds that the sample test error is much higher than the given upper bound on the worst-case Bayes error, two possibilities should be considered.  First, it may imply that the classifier's decision boundary is far from optimal, and that the classifier should be improved. Or, it could be that the test samples used to judge the test error are an unrepresentative set, and that more test samples should be taken to get a useful estimate of the test error.

There are a number of other results regarding the optimization of  different
functionals given moment constraints (e.g.
\cite{Dowson:73,Poor:80,Moulin:05,Tryphon:06,FriedlanderGupta,Dukkipati:06,Lanckreit:02,Csiszar:01,Csiszar:08}).
However, we are not aware of any previous work bounding the maximum
Bayes error given moment constraints. Some related problems are
considered by Antos et al. \cite{Antos:99};  a key difference to
their work is that while we assume moments are given, they
instead take as given iid samples from the class-conditional distributions,
and they then bound the average error of an estimate of the Bayes error.

After some mathematical preliminaries, we give lower bounds for the
maximum Bayes error in Section \ref{sec:results}.
We construct our lower bounds by creating a truncated moment problem.  The existence of
a particular lower bound then depends on the feasibility of  the
corresponding truncated moment problem, which can be checked using
Curto and Fialkow's solutions \cite{Curto:1991vx} (reviewed in the
appendix).   In Section \ref{sec:upperbound}, we show
that the approach of Lanckreit et al. \cite{Lanckreit:02}, which assumes a linear decision boundary, can be extended to provide an upper bound on the maximum Bayes error.  We provide an  illustration of the tightness of these bounds in Section \ref{sec:experiments}, then  end with a discussion and some open questions.

\section{Bayes Error}
Let $\mathcal{X}$ be a vector space and let $\mathcal{Y}$ be a finite set of
classes. Without loss of generality we may assume that
$\mathcal{Y}=\{1,\ldots,G\}$. Suppose that there is a measurable classification
function $h:\mathcal{X}\to S_G$ where $S_G$ is the
$(G-1)$ probability simplex. Then the i$th$ component of
$h(x)$ can be interpreted as the probability of class $i$ given $x$, and we write
$p(i|x) = h(x)_i$.

For a given $x\in\mathcal{X}$, the Bayes classifier selects the class $\hat{y}(x)$ that maximizes the posterior probability $p(i|x)$ (if there is a tie for the maximum, then any of the tied classes can be chosen). The probability that the Bayes classifier is wrong for a given $x$ is
\begin{equation}\label{eqn:1}
P_e(x) = 1 - \max_{i\in\mathcal{Y}} p(i|x).
\end{equation}
Suppose there is a probability measure $\nu$ defined on $\mathcal{X}$. Then the
\emph{Bayes error} is the expectation of $P_e$:
\begin{equation}\label{eqn:maxin}
\E[P_e] =1-\int_{x \in \mathcal{X}}  \max_{i\in\mathcal{Y}} p(i|x)   d\nu(x).
\end{equation}
The fact that the $p(i|x)$ must sum to one over $i$, and thus $\max_i p(i|x) \geq 1/G$,  implies a trivial upper bound on the Bayes error given in (\ref{eqn:maxin}):
\begin{align*}
\E[P_e]  & \le 1 - \int_x  \frac{1}{G} d\nu(x)  \\
& = \frac{G-1}{G}.
\end{align*}

Suppose that the probability measure $\nu$ is defined on $\mathcal{X}$ such that it is absolutely continuous w.r.t. the Lebesgue measure such that it has density $p(x)$. Or suppose that it is discrete and expressed  as
\begin{equation*}%\label{eqn:discretemeasure}
\nu = \sum_{j=1}^\infty \alpha_j\delta_{x_j},
\end{equation*}
where $\delta_{x_j}$ is the Dirac measure with support $x_j$, $\alpha_j>0$ for all
$j=1,2,\ldots$ and $\sum_{j=1}^\infty \alpha_j=1$, and we say the density $p(x_j) = \alpha_j$.
In either case,  (\ref{eqn:maxin}) can be expressed in terms of the $i$th class prior $p(i)=\int p(i|x)d\nu(x)$ and $i$th class-conditional density $p(x |i)$ (or probability mass function $p(x_j| i)$) as follows:
\begin{align}
\E[P_e]
& =\begin{cases}   1- \displaystyle \sum_{j=1}^{\infty} \max_{i \in \mathcal{Y}} p(x_j|i)p(i) & \textrm{ in the discrete case} \\
1- \displaystyle \int \max_{i\in\mathcal{Y}} p(x|i) p(i) dx  & \textrm{ in the absolutely continuous case.} \\
\end{cases}
\label{eqn:bayeserror}
\end{align}

If $\nu$ is a general measure then Lebesgue's decomposition theorem says that it can be
written as a sum of three measures: $\nu=\nu_d+\nu_{ac}+\nu_{sc}$. Here $\nu_d$ is a
discrete measure and the other two measures are continuous; $\nu_{ac}$ is absolutely
continuous w.r.t. the Lebesgue measure, and $\nu_{sc}$ is the remaining singular part. We
have a convenient representation for both the discrete and the absolutely continuous part
of a measure but not for the singular portion. For this reason we are going to restrict
our attention to measures that are either discrete or absolutely continuous (or a linear
combination of these kind of measures).

\section{Lower Bounds for Worst-Case Bayes Error}\label{sec:results}
Our strategy to providing a lower bound on the supremum Bayes error is to
constrain the $G$ probability distributions $p(x|i)$, $i\in \mathcal{Y}$ to have
an overlap of size $\epsilon \in (0,1)$. Specifically, we constrain the $G$
distributions to each have a Dirac measure  of size $\epsilon$ at the same
location. In the case of uniform class prior probabilities this makes the Bayes
error at least $\epsilon\frac{G-1}{G}$.  The largest such $\epsilon$ for which this
overlap constraint is feasible determines the best lower bound on the worst-case Bayes
error this strategy can provide. The maximum such feasible $\epsilon$ can be determined by checking whether there is a solution to a corresponding truncated moment problem (see the appendix for details). Note that this approach does not restrict the distributions from overlapping elsewhere which would increase the Bayes error, and thus this approach only provides a lower bound to the maximum Bayes error.

We first present a constructive solution showing that no matter what the first moments are, the Bayes error can be arbitrarily bad if only the first moments are given. Then we derive conditions for the size of the lower bound for the two moment case and three moment case, and end with what we can say for the general case of $n$ moments.

\begin{lem} Suppose the first moments $\gamma_{1,i}$ are
given for each $i$ in a subset of $\{1, \ldots, G\}$ and the remaining
class-conditional distributions are unconstrained.  Then for all
$1>\epsilon>0$ one can construct G discrete or absolutely continuous
class-conditional distributions such that the Bayes error $\E[P_e]
\ge (1-\max_{i\in\mathcal{Y}}p(i))\epsilon$.
\end{lem}

\begin{proof}
This lemma works for any vector space $\mathcal{X}$. The moment constraints
hold if the $i$th class-conditional distribution is taken to be
$p(x | i) = \epsilon\delta_0(x) +
(1-\epsilon)\delta_{z_i}(x)$ where $z_i =
\frac{\gamma_{1,i}}{(1-\epsilon)}$. This constructive solution exists
for any $\epsilon \in (0,1)$ and yields a Bayes error of at least
$(1-\max_{i\in\mathcal{Y}}p(i))\epsilon$. To see this,
substitute the $G$ measures $\{ \epsilon\delta_0 +
(1-\epsilon)\delta_{z_i}\}$  into (\ref{eqn:bayeserror}) to produce
\begin{align*}
%Maya notes - Bela I changed this from what you had because
% it looked like you weren't taking into account correctly the (trivial) case that
%the means were the same, which would put some of the other
% deltas right on top of each other and make your first equality
% actually an inequality.
 \E[P_e] & =1- \sum_{j=1}^{\infty}  \max_{i\in\mathcal{Y}}\left(  \epsilon \delta_0(x_j) + (1-\epsilon)\delta_{z_i}(x_j) \right) p(i) \\
  &=1-\left(\max_{i\in\mathcal{Y}}p(i)\epsilon +\sum_{j=1}^{\infty}
   \max_{i\in\mathcal{Y}} p(i) (1-\epsilon)\delta_{z_i}(x_j)\right)\\
   &\ge 1- \left(\epsilon\max_{i\in\mathcal{Y}}p(i)+(1-\epsilon)\sum_{i=1}^G
     p(i)\right)\\
   &=\epsilon \left(1-\max_{i\in\mathcal{Y}}p(i)\right).
\end{align*}
For an absolutely continuous example,
consider $\mathcal{X}=\R$. The uniform densities
$p_l(x|i)=\frac{1}{2lp(i)}\mathbb{I}_{[\gamma_{1,i}-(lp(i)),\gamma_{1,i}+(lp(i))]}(x)$
with
$i=1,2,\ldots,G$ (where $\mathbb{I}_E$ is the indicator function of the set
$E$) provide  class-conditional distributions such that as
$l\to\infty$ the Bayes error goes to $1-\max_{i\in\mathcal{Y}}p(i)$. To see this, let
$i^* \in\argmax_{i\in\mathcal{Y}} p(i)$ and consider the difference
$d_i=\gamma_{1,i^*}+lp(i^*)-(\gamma_{1,i}+lp(i))=
(\gamma_{1,i^*}-\gamma_{1,i})+l(p(i^*)-p(i))$. If $p(i^*)-p(i)>0$ then $d_i\to
\infty$ as $l\to \infty$ therefore there is an $l'>0$ such that if $l>l'$ then
$d_i>0$ and hence
$\gamma_{1,i^*}+lp(i^*)>\gamma_{1,i}+lp(i)$. A similar derivation shows that
there is an $l''>0$ such that if $l>l''$ then
$\gamma_{1,i^*}-lp(i^*)<\gamma_{1,i}-lp(i)$. In other words, if $p(i^*)-p(i)>0$
then $p(i^*)p(x|i^*)$ eventually dominates $p(i)p(x|i)$ since all the functions
$p(j)p(x|j)$, $j=1,\ldots,G$ have the same amplitude $\frac{1}{2l}$. If $p(i^*)-p(i)=0$ then the integral of the function $p(i)p(x|i)$ that is not
dominated by $p(i^*)p(x|i^*)$ is $\frac{\abs{\gamma_{1,i^*}-\gamma_{1,i}}}{2l}\to 0$ as $l\to \infty$. Finally, the integral of the dominant function $p(i^*)p(x|i^*)$ is
$\frac{1}{2l}2lp(i^*)=p(i^*)$ and therefore the Bayes error approaches
$1-\max_{i\in\mathcal{Y}}p(i)$ as $l\to \infty$.

\end{proof}

\begin{thm}\label{thm}
Suppose that $\mathcal{X}=\R$ and that there exist $G$ class-conditional measures with
moments $\{\gamma_{1,i},\gamma_{2,i}\}$, $i\in \mathcal{Y}$. Given only this set of moments
$\{\gamma_{1,i},\gamma_{2,i}\}$, $i\in \mathcal{Y}$, a lower bound on the
supremum Bayes error is
\begin{align*}
  \sup \E[P_e] &\ge \sup_{\Delta\in\R} \left[\sum_{i=1}^G
    p(i)\left(1-\frac{(\gamma_{1,i}- \Delta)^2}{\gamma_{2,i}+ \Delta^2 -
      2\Delta \gamma_{1,i}}\right) -\max_{i\in\mathcal{Y}}
    \left\{p(i)\left(1-\frac{(\gamma_{1,i}- \Delta)^2}{\gamma_{2,i}+ \Delta^2 -
      2\Delta \gamma_{1,i}}\right)\right\}\right]\\
&\ge (G-1)\sup_{\Delta\in\R} \min_{i\in\mathcal{Y}}
    \left\{p(i)\left(1-\frac{(\gamma_{1,i}- \Delta)^2}{\gamma_{2,i}+ \Delta^2 -
      2\Delta \gamma_{1,i}}\right)\right\},
\end{align*}
where the supremum on the left-hand-side is taken over all combinations
of $G$ class-conditional measures satisfying the moment constraints.

If $G=2$ then
\begin{equation}\label{eqn:G2case}
  \sup \E[P_e] \geq \sup_{\Delta\in\R} \min_{i=1,2}
    \left\{p(i)\left(1-\frac{(\gamma_{1,i}- \Delta)^2}{\gamma_{2,i}+ \Delta^2 -
      2\Delta \gamma_{1,i}}\right)\right\}.
\end{equation}
Further, if the class priors are equal, then, in terms of the centered second moment  $\sigma_i^2 =
\gamma_{2,i}-\gamma_{1,i}^2$, the optimal $\Delta$ value is one of
\begin{equation}\label{eq:maxdeltaquadratic}
\Delta = \frac{-(\gamma_{1,2}\sigma_1^2-\gamma_{1,1}\sigma_2^2) \pm
\sigma_1\sigma_2 \abs{ \gamma_{1,1}-\gamma_{1,2}}}{\sigma_2^2-\sigma_1^2}
\end{equation}
if $\sigma_1\neq \sigma_2$.  Otherwise, if $\sigma_1= \sigma_2$,
\begin{equation}\label{eq:maxdeltalinear}
\Delta=\frac{\gamma_{1,1}+\gamma_{1,2}}{2},
\end{equation}
and the lower bound simplifies:
\begin{equation*}
\sup \E[P_e] \geq \frac{2\sigma_1^2}{4\sigma_1^2+(\gamma_{1,1}-\gamma_{1,2})^2}.
\end{equation*}
\end{thm}

\begin{proof}
Consider some $0 < \epsilon < 1$. If the class prior is uniform then a
sufficient condition for the Bayes
error to be at least $\frac{G-1}{G}\epsilon$ is if  all of the unknown
measures share a Dirac measure of at least $\epsilon$. First, we place this Dirac measure at zero and find the maximum $\epsilon$ for which this can be done. Then later in the proof we show that a larger  $\epsilon$ (and hence a tighter lower bound on the maximum Bayes error) can be found by placing this shared Dirac measure in a more optimal location, or equivalently, by shifting all the measures.

Suppose a probability measure $\mu$ can be expressed in the form
$\epsilon\delta_0+\tilde{\mu}$ where  $\tilde{\mu}$ is some measure such that $\tilde{\mu}(\{0\})=0$. If $\mu$ satisfies the
original moment constraints then $\tilde{\mu}$ also satisfies them; this follows directly
from the moment definition for $n\ge 1$:
\begin{equation*}%\label{eqn:moments}
  \int x^n d\mu(x) = 0^n \epsilon +  \int x^n d\tilde{\mu}(x) = \int x^n d\tilde{\mu}(x).
\end{equation*}

  % We construct a shift-dependent lower bound on the Bayes error of
  % $\epsilon/2$ by forcing each distribution to have $\epsilon$ Dirac
  % measure at $x = 0$, and then checking to see if each class's moments
  % constraints can still be satisfied. If they can, then $\epsilon/2$
  % is a legitimate lower bound on the Bayes error.
Also $\tilde{\mu}(\mathcal{X})=1-\epsilon$. Thus, we require
a measure $\tilde{\mu}$ with a zeroth moment $\gamma_0=1-\epsilon>0$ and the original first
and
second moments $\gamma_1, \gamma_2$.
Then, as described in Section \ref{sec:tm}, there are two conditions
that we have to check. In order to have a measure
with the prescribed moments, the matrix
  \[
  A=A(1)=
  \begin{bmatrix}
    1-\epsilon &\gamma_1\\
    \gamma_1 &\gamma_2
  \end{bmatrix}.
  \]
  has to be positive semidefinite, which holds if and only if $\epsilon \le
  1-\frac{\gamma_1^2}{\gamma_2}$. (Note that the Theorem assumes that there exists a distribution with the given moments, and thus the above implies that $\gamma_2 \geq \gamma_1^2$).  Moreover, the rank of matrix $A$ and
  the rank of $\gamma$ (for notation see Section \ref{sec:tm}) have to
  be the same. Matrix $A$ can have rank 1 or 2. If $\rank(A)=1$ then
  the columns of $A$ are linearly dependent and therefore
  $\rank(\gamma)=1$. If $\rank(A)=2$ then $A$ is invertible
  and $\rank(\gamma)=2$. Thus there is a measure $\tilde{\mu}$ with moments
  $\{1-\epsilon,\gamma_1,\gamma_2\}$ iff $0\le \epsilon\le
  1-\frac{\gamma_1^2}{\gamma_2}$. If such a $\tilde{\mu}$ exists, then there
 also exists a discrete probability measure with moments $\{1-\epsilon,\gamma_1,\gamma_2\}$ and $0\le \epsilon\le
  1-\frac{\gamma_1^2}{\gamma_2}$ by Curto and Fialkow's Theorem 3.1 and Theorem 3.9 \cite{Curto:1991vx}. 
  
Suppose we have $G$ such discrete probability measures satisfying the
corresponding moments constraints given in the statement of this
theorem. Denote the i{\it th} discrete probability measure by
$\nu_i = \epsilon_i\delta_0+\sum_{j= 1}^{\infty} \alpha_{j,i} \delta_{x_j}$ where $0 \le \epsilon_i\le  1-\frac{\gamma_{1,i}^2}{\gamma_{2,i}}$ and $x_j \neq 0$ for all $j$, and  $j$ indexes the set of all non-zero atoms in the $G$ discrete measures $\{\nu_i\}$. Then the supremum Bayes error is bounded below by the Bayes error for this set of discrete measures:
\begin{align}
\sup \E[P_e]  &\ge 1-\max_{i\in\mathcal{Y}} \{\epsilon_ip(i)\} -  \sum_{j = 1}^{\infty}  \max_{i \in \mathcal{Y}} \alpha_{j,i}  p(i)  \\
&\ge 1-\max_{i\in\mathcal{Y}} \{\epsilon_ip(i)\} -  \sum_{j = 1}^{\infty}  \sum_{i=1}^G \alpha_{j,i} p(i)  \\
&= 1-\max_{i\in\mathcal{Y}} \{\epsilon_ip(i)\} -  \sum_{i=1}^G  p(i) \sum_{j = 1}^{\infty}  \alpha_{j,i} \\
&= 1-\max_{i\in\mathcal{Y}} \{\epsilon_ip(i)\}- \sum_{i=1}^G
p(i)(1-\epsilon_i)\nonumber \\
&=\sum_{i=1}^G p(i)\epsilon_i-\max_{i\in\mathcal{Y}} \{\epsilon_ip(i)\}.\label{eq:maxrhs}
\end{align}

This is true for any collection of $\epsilon_i$, $\epsilon_i\le
  1-\frac{\gamma_{1,i}^2}{\gamma_{2,i}}$. This means that $\sup \E[P_e]$ is an
  upper bound for (\ref{eq:maxrhs}) for these admissible $\epsilon_i$ and we can
  find a tighter inequality by finding the supremum of (\ref{eq:maxrhs}) over the
  set of admissible $\epsilon_i$. The domain of the function (\ref{eq:maxrhs}) is the Cartesian product
  $\prod_{i=1}^G \left[0, 1- \frac{\gamma_{1,i}^2}{\gamma_{2,i}}\right]$. It is a non-empty
  compact set and (\ref{eq:maxrhs}) is continuous, so we can expect to find a
  maximum. The maximum is unique and to find it let $(\epsilon_1,\ldots,\epsilon_G)$ be
  any element in the domain and let $i^*\in
  \argmax_i\left\{p(i)\left(1-\frac{\gamma_{1,i}^2} {\gamma_{2,i}}\right)\right\}$. Since
  $p(i)\ge 0$, we have
\begin{align*}
\sum_{i=1}^G p(i)\epsilon_i-\max_{i\in\mathcal{Y}}\{p(i)\epsilon_i\} &\le \sum_{i\neq  i^*}p(i)\epsilon_i \\
&\le \sum_{i\neq i^*}p(i)\left(1-\frac{\gamma_{1,i}^2}{\gamma_{2,i}}\right)\\
&= \sum_{i=1}^G p(i)\left(1-\frac{\gamma_{1,i}^2}{\gamma_{2,i}}\right) -\max_{i\in\mathcal{Y}}
    \left\{p(i)\left(1-\frac{\gamma_{1,i}^2}{\gamma_{2,i}}\right)\right\}.
\end{align*}
Therefore
\begin{subequations}
\begin{align}
\sup \E[P_e] &\ge \sum_{i=1}^G p(i)\left(1-\frac{\gamma_{1,i}^2}{\gamma_{2,i}}\right)
-\max_{i\in\mathcal{Y}}\left\{p(i)\left(1-\frac{\gamma_{1,i}^2}{\gamma_{2,i}}\right)\right\}
\nonumber \\
&\ge (G-1)\min_{i\in\mathcal{Y}}\left\{p(i)
  \left(1-\frac{\gamma_{1,i}^2}{\gamma_{2,i}}\right)\right\}, \label{eqn:shiftylower bound}
\end{align}
\end{subequations}
where the supremum on the left-hand-side is taken over all the
combination of class-conditional measures that satisfy the given
moments constraints.

%  $i^*=\argmax_i
%   \left\{1-\frac{\gamma_{1,i}^2}{\gamma_{2,i}}\right\}$. Then
% \[
% \sum_{i=1}^G \epsilon_i-\max_{i\in\mathcal{Y}}\epsilon_i\le \sum_{i\neq i^*}\epsilon_i\le
% \sum_{i\neq i^*} \left(1-\frac{\gamma_{1,i}^2}{\gamma_{2,i}}\right)= \sum_{i=1}^G
% \left(1-\frac{\gamma_{1,i}^2}{\gamma_{2,i}}\right) -\max_{i\in\mathcal{Y}}
% \left\{1-\frac{\gamma_{1,i}^2}{\gamma_{2,i}}\right\}.
% \]
% Therefore
% \begin{subequations}
%   \begin{align}
% \sup \E[P_e] &\ge \frac{1}{G}\left[\sum_{i=1}^G
%         \left(1-\frac{\gamma_{1,i}^2}{\gamma_{2,i}}\right)
%         -\max_{i\in\mathcal{Y}}
%         \left\{1-\frac{\gamma_{1,i}^2}{\gamma_{2,i}}\right\}\right] \nonumber \\
%       &\ge \frac{G-1}{G}\min_{i\in\mathcal{Y}}
%       \left\{1-\frac{\gamma_{1,i}^2}{\gamma_{2,i}}\right\},\label{eqn:shiftylower bound}
%     \end{align}
%   \end{subequations}

The next step follows from the fact that the Lebesgue measure
and the counting measure are shift-invariant measures and the Bayes
error is computed by integrating some functions against those measures. Suppose we had $G$ class distributions, and we shift each of them by $\Delta$. The Bayes error would not change. However, our lower bound given in (\ref{eqn:shiftylower bound}) depends
on the actual given means $\{\gamma_{1,i}\}$, and in some cases we can produce
a better lower bound by shifting the distributions before applying the above lower bounding strategy. The shifting approach we present next is equivalent to placing the shared $\epsilon$ measure someplace other than at the origin.

Shifting a distribution by
$\Delta$ does change all of the moments (because they are not centered
moments), specifically, if $\mu$ is a probability measure with finite
moments $\gamma_0 = 1$, $\gamma_1$, \ldots, $\gamma_n$, and
$\mu_{\Delta}$ is the measure defined by $\mu_{\Delta}(D) = \mu(D +
\Delta)$ for all $\mu$-measurable sets $D$, then the $n$-th
non-centered moment of the shifted measure $\mu_{\Delta}$ is
\begin{equation*}%\label{eqn:shiftedMoments}
  \tilde{\gamma}_n = \int x^n d\mu_{\Delta}(x) = \int (x-\Delta)^n
  d\mu(x) = \sum_{k=0}^n (-1)^{n-k}\binom{n}{k}\Delta^{n-k} \gamma_k,
\end{equation*}
where the second equality can easily be proven for any $\sigma$-finite
measure using the definition of integral. This same formula shows that
shifting back the measure will transform back the moments.

For the two-moment case,  the shifted
measure's moments are related to the original moments by:
\begin{align*}
  \tilde{\gamma}_1 &= \gamma_1 - \Delta\\
  \tilde{\gamma}_2 &= \gamma_2 + \Delta^2 - 2\Delta\gamma_1.
\end{align*}

Then a tighter lower bound can be produced by choosing the shift
$\Delta$ that maximizes the shift-dependent lower bound given in
(\ref{eqn:shiftylower bound}):
\begin{align*}
  \sup \E[P_e] &\ge \sup_{\Delta\in\R} \left[\sum_{i=1}^G
    p(i)\left(1-\frac{(\gamma_{1,i}- \Delta)^2}{\gamma_{2,i}+ \Delta^2 -
      2\Delta \gamma_{1,i}}\right) -\max_{i\in\mathcal{Y}}
    \left\{p(i)\left(1-\frac{(\gamma_{1,i}- \Delta)^2}{\gamma_{2,i}+ \Delta^2 -
      2\Delta \gamma_{1,i}}\right)\right\}\right]\\
&\ge (G-1)\sup_{\Delta\in\R} \min_{i\in\mathcal{Y}}
    \left\{p(i)\left(1-\frac{(\gamma_{1,i}- \Delta)^2}{\gamma_{2,i}+ \Delta^2 -
      2\Delta \gamma_{1,i}}\right)\right\}.
\end{align*}

If $G=2$ then this lower bound is
\begin{multline*}
  \sup_{\Delta\in\R} \left[\sum_{i=1}^2
    p(i)\left(1-\frac{(\gamma_{1,i}- \Delta)^2}{\gamma_{2,i}+ \Delta^2 -
        2\Delta \gamma_{1,i}}\right) -\max_{i\in\mathcal{Y}}
    \left\{p(i)\left(1-\frac{(\gamma_{1,i}- \Delta)^2}{\gamma_{2,i}+ \Delta^2 -
        2\Delta \gamma_{1,i}}\right)\right\}\right]\\
  = \sup_{\Delta\in\R} \min_{i=1,2}
  \left\{p(i)\left(1-\frac{(\gamma_{1,i}- \Delta)^2}{\gamma_{2,i}+ \Delta^2 -
      2\Delta \gamma_{1,i}}\right)\right\}.
\end{multline*}

We can say more in the case of equal class priors, that is, if $p(1)=p(2)=1/2$. The
functions \[
f_i(\Delta)=1-\frac{(\gamma_{1,i}-
  \Delta)^2}{\gamma_{2,i}+ \Delta^2 - 2\Delta \gamma_{1,i}}
\]
are maximized at  $\Delta = \gamma_{1,i}$,  where the maximum value is $1$ and the derivative of
function $f_i$ is strictly positive
for $\Delta<\gamma_{1,i}$ and strictly negative for
$\Delta>\gamma_{1,i}$, $i=1,2$. This means that the potential maximum occurs at the point
where the two functions are equal. This results in a quadratic equation if
$\sigma_2^2 \neq \sigma_1^2$ with
solutions (\ref{eq:maxdeltaquadratic}),  and otherwise a linear one with solution (\ref{eq:maxdeltalinear}).

If $\gamma_{1,1}=\gamma_{1,2}$ then the function $f_i$ with smaller $\gamma_{2,i}$ will
provide us with the lower bound which is $1/2$, as expected. If $\gamma_{1,1}\neq
\gamma_{1,2}$ then since $f_i(\gamma_{1,i})=1$ the maximum occurs at a $\Delta$ value which
is between the two $\gamma_{1,i}$. To see this let $J$ be the interval defined by the two
$\gamma_{1,i}$. As a consequence of the strict nature of the derivatives, for any $\Delta$
value outside of the interval $J$ the function
\[
\min_{i\in\mathcal{Y}}\left\{1-\frac{(\gamma_{1,i}- \Delta)^2}{\gamma_{2,i}+ \Delta^2 -
    2\Delta \gamma_{1,i}}\right\}
\]
is less than on $J$. But on $J$ the function $f_1(\Delta)-f_2(\Delta)$ is continuous and
thanks to the fact that $f_i(\gamma_{1,i})=1$ and the behavior of the derivatives, it has
different sign at the two endpoints of $J$. This means that there is a $\Delta\in J$ such
that $f_1(\Delta)-f_2(\Delta)=0$.
\end{proof}

This theorem applies only to one-dimensional distributions. The approach of constraining the distributions to have measure $\epsilon$ at a common location can be extended to higher-dimensions, but actually determining whether the moment constraints can still be satisfied becomes significantly hairier; see \cite{Curto:1991vx} for a sketch of the truncated moment solutions for higher dimensions.

% \[\Delta_1=-\frac
% {\gamma_{1,2}\gamma_{2,1}-\gamma_{1,1}^2\gamma_{1,2}+\gamma_{1,2}^2\gamma_{1,1}-\gamma_{2,2}\gamma_{1,1}+
%   \sqrt{ \left( \gamma_{1,2}^2-\gamma_{2,2} \right)  \left(
%       \gamma_{1,1}^2-\gamma_{2,1} \right)  \left(
%       -\gamma_{1,2}+\gamma_{1,1} \right) ^2}} {\gamma_{2,2}+\gamma_{1,1}^2-\gamma_{2,1}-\gamma_{1,2}^2}\]
% and
% \[\Delta_2=\frac
% {-\gamma_{1,2}\gamma_{2,1}+\gamma_{1,1}^2\gamma_{1,2}-\gamma_{1,2}^2\gamma_{1,1}+\gamma_{2,2}\gamma_{1,1}+
%   \sqrt{ \left( \gamma_{1,2}^2-\gamma_{2,2} \right)  \left(
%       \gamma_{1,1}^2-\gamma_{2,1} \right)  \left(
%       -\gamma_{1,2}+\gamma_{1,1} \right) ^2}} {\gamma_{2,2}+\gamma_{1,1}^2-\gamma_{2,1}-\gamma_{1,2}^2}\]
% If we have more than one potential
% solution, after checking the derivatives of the functions
% \[
% f_i'(\Delta)=\frac{2(\Delta-\gamma_{1,i})(\gamma_{1,i}^2-\gamma_{2,i})}
% {(\gamma_{2,i}+\Delta^2-2\Delta\gamma_{1,i})^2}
% \]
% at the potential $\Delta$ values, we can determine which one gives the maximum.

An argument similar to the one given in the last two paragraphs of the previous proof
can be used to show that if  $\gamma_{2,i}-\gamma_{1,i}^2$ are all equal for all
$i\in \mathcal{Y}$  and any finite $G$,  and if the class priors are equal,
then the optimal $\Delta$ is
$\frac{\gamma_{1,\min}+\gamma_{1,\max}}{2}$, where
$\gamma_{1,\min}$ and $\gamma_{1,\max}$ are the smallest and the largest values in the set
$\{\gamma_{1,i}\}_{i\in \mathcal{Y}}$, respectively. To see this, we start with rewriting the
function $f_i$:
\[
f_i(\Delta)=1-\frac{(\gamma_{1,i}-
  \Delta)^2}{\gamma_{2,i}+ \Delta^2 - 2\Delta \gamma_{1,i}}=
\frac{\sigma_i^2}{\sigma_i^2+(\Delta-\gamma_{1,i})^2}.
\]
This shows, that if the condition mentioned above holds then the functions $f_i$ are
shifted versions of each other. Let $f_{\min}$ and $f_{\max}$ be the functions corresponding
to $\gamma_{1,\min}$ and $\gamma_{1,\max}$, respectively and let $\Delta'$ be the point
where $f_{\min}$ and $f_{\max}$ intersect. Because of the strict nature of the derivatives of
$f_i$ and because the functions are shifted versions of each other, for any $\Delta\ge
\Delta'$,  $f_{\min}$ is smaller than any other $f_i$. Because of symmetry, it is
true that for any $\Delta\le \Delta'$, $f_{\max}$ is smaller than any other
$f_i$. Again, by symmetry we have that $\Delta'=\frac{\gamma_{1,\min}+\gamma_{1,\max}}{2}$
and therefore this is the optimal $\Delta$.

\begin{cor}\label{cor3moments}
Suppose that $\mathcal{X}=\R$ and that the first, the second and the third moments are
given for $G$
class-conditional measures, i.e. for the i{\it th} class-conditional
measure we are given
$\{\gamma_{1,i},\gamma_{2,i},\gamma_{3,i}\}$. Then the Bayes error has lower bound
\begin{align*}
  \sup \E[P_e] &\ge  \sup_{\delta > 0} \sup_{\Delta\in\R} \left[\sum_{i=1}^G
    p(i)\left(1-\frac{(\gamma_{1,i}- \Delta)^2}{\gamma_{2,i}+ \Delta^2 -
      2\Delta \gamma_{1,i}}-\delta\right) -\max_{i\in\mathcal{Y}}
    \left\{p(i)\left(1-\frac{(\gamma_{1,i}- \Delta)^2}{\gamma_{2,i}+ \Delta^2 -
      2\Delta \gamma_{1,i}}-\delta\right)\right\}\right]\\
&\ge (G-1) \sup_{\Delta\in\R} \min_{i\in\mathcal{Y}}
    \left\{p(i)\left(1-\frac{(\gamma_{1,i}- \Delta)^2}{\gamma_{2,i}+ \Delta^2 -
      2\Delta \gamma_{1,i}}\right)\right\}.
\end{align*}
\end{cor}
\begin{proof}
  In this case we have a list of four numbers
  $\{1-\epsilon,\gamma_1,\gamma_2,\gamma_3\}$ and again
  \[
  A=A(1)=
  \begin{bmatrix}
    1-\epsilon &\gamma_1\\
    \gamma_1 &\gamma_2
  \end{bmatrix}.
  \]
  If $\gamma_0=1-\epsilon>0$ then $A$ is positive definite
  if $\epsilon\le 1-\frac{\gamma_1^2}{\gamma_2}-\delta<1-\frac{\gamma_1^2}{\gamma_2}$. In this case
  $\mathbf{v}_{2}=(\gamma_2,\gamma_3)^T$ and it is in the range of $A$
  since $A$ is invertible. The statements in Section \ref{sec:tm}
  imply that there is a measure with moments
  $\{1-\epsilon,\gamma_1,\gamma_2,\gamma_3\}$ and consequently that
\begin{align*}
      \sup \E[P_e] &\ge  \sup_{\delta > 0} \left[\sum_{i=1}^G
        p(i)\left(1-\frac{\gamma_{1,i}^2}{\gamma_{2,i}}-\delta\right)
        -\max_{i\in\mathcal{Y}}
        \left\{p(i) \left(1-\frac{\gamma_{1,i}^2}
            {\gamma_{2,i}}-\delta\right)\right\} \right]\\
      &\ge (G-1)\min_{i\in\mathcal{Y}}
      \left\{p(i)\left(1-\frac{\gamma_{1,i}^2}{\gamma_{2,i}}\right)\right\},
\end{align*}
The rest of the proof follows
analogously to the proof of Theorem \ref{thm}.
\end{proof}

The proof of Corollary \ref{cor3moments} relies on the fact that for $\delta>0$
the matrix $A(1)$ featured in the proof is invertible, so one of the conditions
for the existence of a measure with the given moments is automatically satisfied
(see Appendix). If $\delta=0$ then $A(1)$ is only positive semidefinite and it
is not obvious that the vector $\mathbf{v}_2$ is in the range of $A(1)$.

The following lemma is stated for completeness.
\begin{lem}\label{lemsketch}
Suppose that $\mathcal{X}=\R$ and that the first $n$ moments are given for $G$ equally likely
class-conditional measures, i.e. for the i{\it th} class-conditional
measure we are given
$\{\gamma_{1,i},\gamma_{2,i},\ldots, \gamma_{n,i}\}$. Then if there
exist measures of the form $\epsilon_i\delta_0+\nu_i$ where $\nu_i$
satisfies the moments conditions given above the corresponding Bayes
error can be bounded from below:
\[
\sup \E[P_e] \ge
      1-\frac{1}{G}\left[\max_{i\in\mathcal{Y}}\{\epsilon_i\}+\sum_{i=1}^G
        (1-\epsilon_i)\right]= \frac{1}{G}\left[\sum_{i=1}^G
        \epsilon_i-\max_{i\in\mathcal{Y}}\{\epsilon_i\}\right],
\]
where the supremum on the left is taken over all the measures satisfying the
moment constraints noted above.
% for any small $\delta>0$
% then one can construct atomic class-conditional distributions with
% Bayes error $E[P_e] \geq \frac{G-1}{G}\epsilon$ if there exists distributions with
% total measure $1- \epsilon$ and shifted moments given in (\ref{eqn:shiftedMoments}) for any shift $\Delta$.
\end{lem}

\begin{proof}
The first part of the proof of Theorem \ref{thm} is applicable in this
case.
\end{proof}
As in the case for two moments, the lower bound can be further tightened by optimizing over all
possible shifts of the overlap Dirac measure.

\section{Upper Bound for Worst-Case Bayes Error}\label{sec:upperbound}
Because the Bayes error is the smallest error over all decision boundaries, one approach
to constructing an upper bound on the maximum Bayes error is to restrict the set of
considered decision boundaries to a set for which the worst-case error is easier to
analyze. For example, Lanckreit et al. \cite{Lanckreit:02} take as given the first and
second moments of each class-conditional distribution, and attempt to find the linear
decision boundary classifier that minimizes the worst-case classification error rate with
respect to any choice of class-conditional distributions that satisfy the given moment
constraints. Here we show that this approach can be extended to produce an
upper bound on the supremum Bayes error for the $G=2$ case.

Let $\mathcal{X}$ be any feature space. Suppose one has two fixed class-conditional
measures $\nu_1,\nu_2$ on $\mathcal{X}$. As in Lanckreit et al. \cite{Lanckreit:02},
consider the set of linear decision boundaries. Any linear decision boundary splits the domain into two half-spaces $S_1$ and $S_2$. We work with linear decision boundaries because these are the
only kind of decision boundaries that split the domain into two convex subsets. The error produced by a linear decision boundary corresponding to the split $(S_1, S_2)$ is
\begin{equation}\label{eqn:ldberror}
p(1)  \nu_1(S_2) + p(2) \nu_2(S_1)\geq \E[P_e](\nu_1,\nu_2).
\end{equation}
That is, the error from any linear decision boundary upper bounds the Bayes error $\E[P_e]$ for  two given measures. To obtain a tighter upper bound on the Bayes error, minimize the left-hand side over all linear decision boundaries:
\begin{equation}\label{eqn:ldberrorTight}
\inf_{S_1,S_2}  \left( p(1)\nu_1(S_2) + p(2) \nu_2(S_1)\right) \geq \E[P_e](\nu_1,\nu_2).
\end{equation}

Now suppose $\nu_1$ and $\nu_2$ are unknown, but their first moments (means) and second
centered moments $(\mu_1, \Sigma_1)$ and $(\mu_2, \Sigma_2)$ are given. Then we note that
the supremum over all measures $\nu_1$ and $\nu_2$ with those moments of the
smallest linear decision boundary error forms an upper bound on the supremum Bayes error where the supremum is taken with respect to the feasible measures $\nu_1,\nu_2$:
\begin{align}
\sup \E[P_e] &\leq \sup_{\nu_1| \mu_1, \Sigma_1} \sup_{\nu_2| \mu_2, \Sigma_2}  \inf_{S_1,S_2}   \left( p(1)\nu_1(S_2) +  p(2) \nu_2(S_1) \right)  \label{eqn:upper1}\\
& \leq  \inf_{S_1,S_2}\left( p(1)\sup_{\nu_1| \mu_1, \Sigma_1}
  \nu_1(S_2) +  p(2)\sup_{\nu_2| \mu_2, \Sigma_2}  \nu_2(S_1) \right). \label{eqn:upper2}
\end{align}

This upper bound can be simplified using the following result\footnote{Some readers may recognize this result as a strengthened and generalized version of the Chebyshev-Cantelli inequality.} by Bertsimas and Popescu \cite{Bertsimas:2005fe} (which follows from a result by Marshall and Olkin \cite{Marshall:60}):
\begin{equation}\label{eqn:marshall}
\sup_{\nu|\mu, \Sigma} \nu(S) = \frac{1}{1 + c(S)} \textrm{ where } c(S) = \inf_{x
  \in S} (x - \mu)^T\Sigma^{-1}(x - \mu),
\end{equation}
where the $\sup$ in (\ref{eqn:marshall}) is over all probability measures $\nu$ with
domain $\mathcal{X}$, mean $\mu$, and centered second moment $\Sigma$; and $S$ is some convex set
in the domain of $\nu$.

%In words, (\ref{eqn:marshall}) says, ``Over all possible probability distributions with the given mean and covariance, the largest probability that a random sample from $\nu$ will be in the convex set $S$ is $1/(1+c(S))$."

Since $S_1$ and $S_2$ in (\ref{eqn:upper2}) are half-spaces, they are convex and
(\ref{eqn:marshall}) can be used to quantify the upper bound. For the rest of this section
let $\mathcal{X}$ be one dimensional. Then the covariance matrices $\Sigma_1$
and $\Sigma_2$ are just scalars that we denote by $\sigma_1^2$ and $\sigma_2^2$, respectively. In one-dimension, any decision boundary that results in a half-plane split is simply a point  $s \in \R$. Without loss of generality with respect to the Bayes error,
let $\mu_1 = 0$ and $\mu_1 \leq \mu_2$.   Then  $c(S_1)$ and $c(S_2)$ in
(\ref{eqn:marshall}) simplify (for details, see Appendix A of \cite{Lanckreit:02}), so
that (\ref{eqn:upper2}) becomes
\begin{equation}
\sup \E[P_e] \leq \min\left\{\inf_{s} \left(
    \frac{p(1)}{1 + \frac{s^2}{\sigma_1^2}}  + \frac{p(2)}{1 + \frac{(\mu_2 -s)^2}{\sigma_2^2}}
  \right),1\right\}. \label{eqn:upper3}
\end{equation}

If $\sigma_1=\sigma_2$ and $p(1) = p(2) = 1/2$, then the infimum occurs at $s = \mu_2/2$ and the upper bound
becomes
\begin{equation*}
\sup \E[P_e]  \leq \min\left\{\frac{4}{4+\frac{\mu_2^2}{\sigma_1^2}} ,
  \frac{1}{2}\right\}.
\end{equation*}
For this case, the given upper bound is twice the given lower bound.

\section{Comparison to Error With Gaussians}\label{sec:experiments}
We illustrate the bounds described in this paper for the common case that the first two moments are
known for each class, and the classes are equally likely.  We compare with the Bayes error produced under the assumption that the distributions are Gaussians with the given moments. In both cases the first distribution's mean is 0 and the variance is 1, and the second distribution's mean is
varied from 0 to 25 as shown on the x-axis. The second distribution's variance is 1 for the comparison shown in the top of  Fig. \ref{fig:lowerbound}. The second distribution's
variance is 5 for the comparison shown in the bottom of Fig. \ref{fig:lowerbound}. For the first case, $\sigma_1 = \sigma_2$ so the infimum in (\ref{eqn:upper3}) occurs at $s = \mu_2/2$ and the upper bound is $\min\left\{4/(4 + \mu_2^2),\frac{1}{2}\right\}$. For the second case with different variances we compute (\ref{eqn:upper3}) numerically.

Fig. \ref{fig:lowerbound} shows that the Bayes error produced by the Gaussian assumption is optimistic
compared to the given lower bound for the worst-case (maximum) Bayes error. Further, the difference
between the Gaussian Bayes error and the lower bound is much larger in
the second case when the variances of the two distributions differ.

\begin{figure}[t]
 \centering
 \begin{tabular}{c}
\includegraphics[width=.5\textwidth]{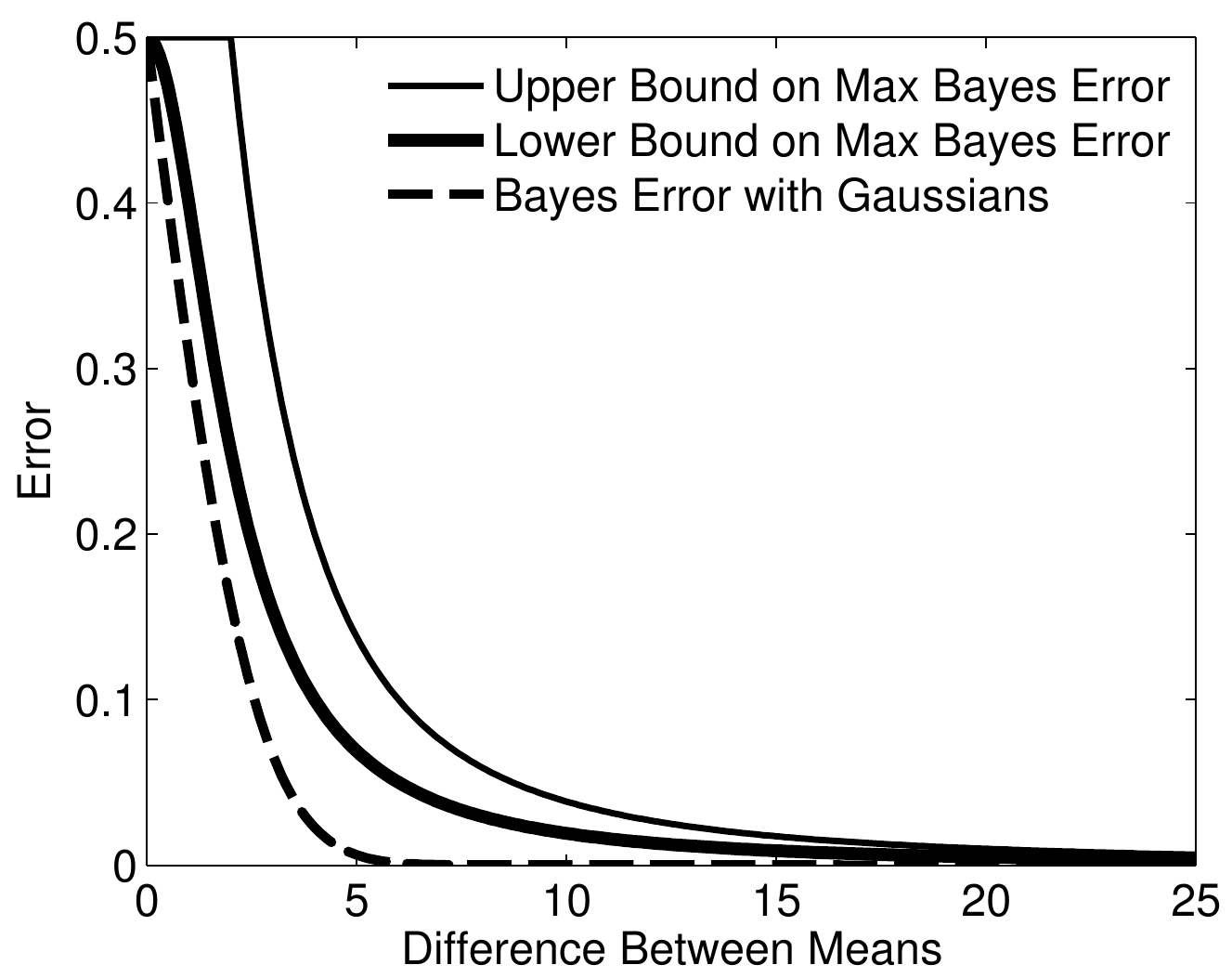}\\
 $\sigma_2^2=1$\\
\\
\includegraphics[width=.5\textwidth]{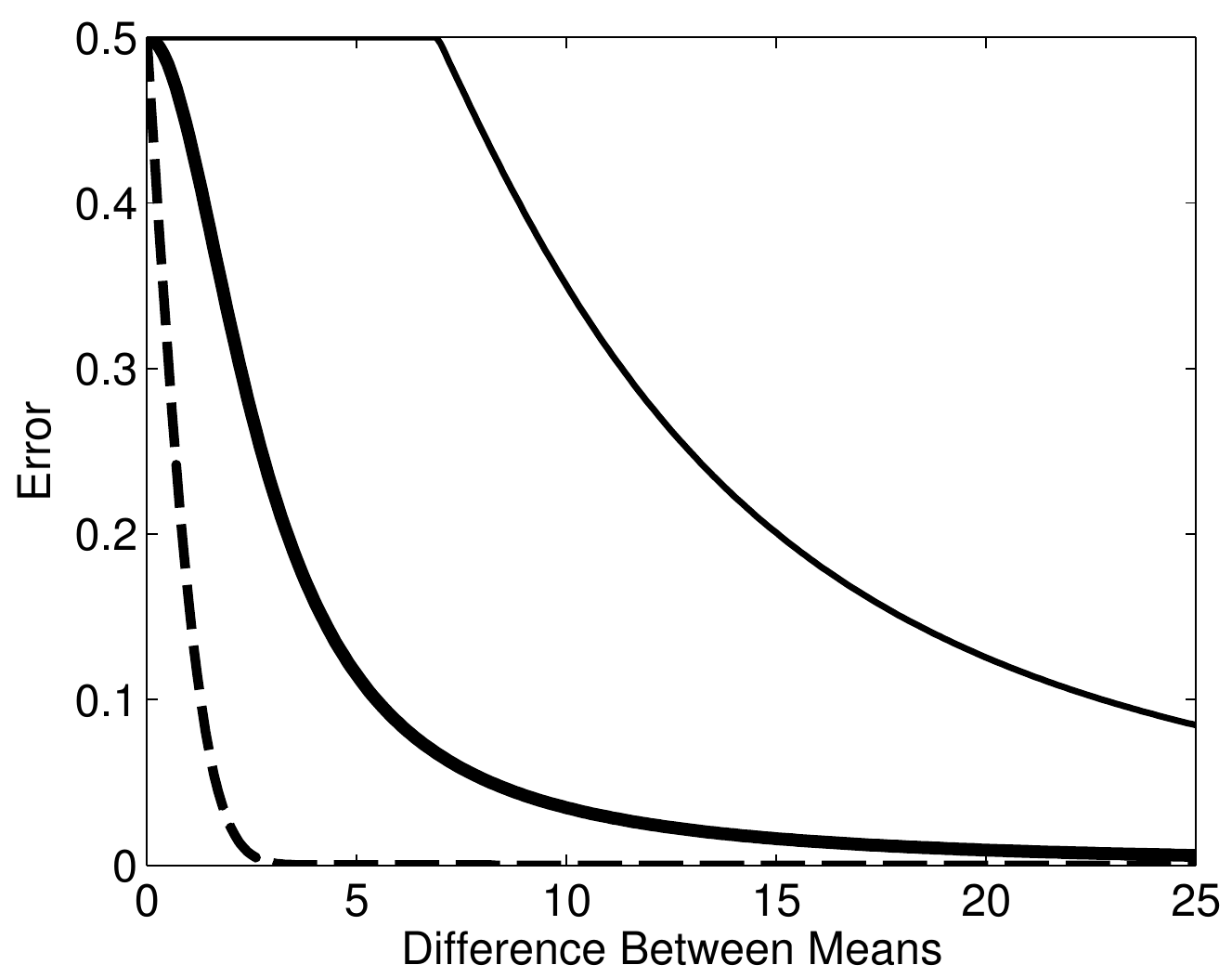}\\
$\sigma_2^2=5$\\
\\
\end{tabular}
 \caption{Comparison of the given lower bound for the worst-case
Bayes error with the Bayes error produced by Gaussian
class-conditional distributions.}
 \label{fig:lowerbound}
\end{figure}

\section{Discussion and Some Open Questions}
We have provided a lower and upper bound on the worst-case Bayes error, but a number of open questions arise from this work.

Lower bounds for the worst-case Bayes error can be constructed by constraining the distributions. We have shown that constraining the distributions to be Gaussians produces a weak lower bound, and we provided a tighter lower bound by constraining the distributions to overlap in a Dirac measure of $\epsilon$.  Given only first moments, our lower bound is tight in that it is arbitrarily close to the worst
possible Bayes error. Given two moments, we have shown that the
common QDA Gaussian assumption for class-conditional distributions is
much more optimistic than our lower bound and increasingly optimistic for increased difference between the variances. However, because in our constructions we do not control all the possible overlap between the
class-conditional distributions, we believe it should be possible to
construct tighter lower bounds.

On the other hand, upper bounds on the worst-case Bayes error can be constructed by
constraining the considered decision boundaries. Here, we considered an upper bound
resulting from restricting the decision boundary to be linear. For the two moment case, we
have shown that work by Lanckreit et al. leads almost directly to an upper bound. However,  the inequality we had to introduce in (\ref{eqn:upper2}) when we switched the
$\inf$ and $\sup$ may make this upper bound loose. It remains an open question if there are conditions under which the upper bound is tight.

Our result that the popular Gaussian assumption is generally not very robust in terms of  worst-case Bayes error prompts us to question whether there are other distributions that are mathematically or computationally convenient to use in generative classifiers that would have a Bayes error closer to the given lower bound.

In practice, a moment constraint is often created by estimating the moment from samples drawn iid from the distribution. In that case, the moment constraint need not be treated as a hard constraint as we have done here. Rather, the observed samples can imply a probability distribution over
the moments, which in turn could imply a distribution over corresponding
bounds on the Bayes error. A similar open question is a sensitivity analysis of how changes in the moments would affect the bounds.

Lastly, consider the opposite problem:  given constraints on the first  $n$ moments for each of the class-conditional distributions, how small could the Bayes error be? It is tempting to suppose that one could generally find discrete measures that overlapped nowhere, such that the Bayes error was zero. However, the set of measures which satisfy a set of moment constrains may be nowhere dense, and that impedes us from being able to make such a guarantee. Thus, this remains an open question.

\section*{Appendix}

% \section{Simplifying the Bayes Error}
% We detail how (\ref{eqn:maxin}) becomes (\ref{eqn:maxincont}); the discrete argument is analogous. Let $A_i$ be the set of all the $x$ values such that $\hat{y}(x)=i$,
% i.e. $A_i=\{x:\hat{y}(x)=i\}$. Then
% \[
% \hat{y}(x)=\sum_{i=1}^G \mathbb{I}_{A_i}(x)i,
% \]
% and
% \[
% \E[P_e] =1-\sum_{i=1}^G \int_{A_i}p(i|x)d\nu(x).
% \]

% In this case there is a density associated with the measure and
% $d\nu(x)=p(x)dx$ therefore
% \[
% p(i|x)d\nu(x)=p(i|x)p(x)dx
% \]
% is the measure (accepted slight abuse of language, see for example
% \cite{Folland},pp. 89) on $\mathcal{X}\times\mathcal{Y}$ with density
% $p(x,i)=p(i|x)p(x)$.   Bayes' theorem says that $p(i|x)p(x)=p(x|i)p(i)$
% where $p(i)=\int p(i|x)p(x)dx$. This means that
% \[
% \E[P_e] =1-\sum_{i=1}^G \int_{A_i}p(i|x)d\nu(x)=1-\sum_{i=1}^G p(i)\int_{A_i}p(x|i)dx.
% \]
% Assume that $p(i)=p(j)=1/G$ for all $i,j\in\mathcal{Y}$
% then for all $x\in A_i$ and for all $k\in\mathcal{Y}$
% \[
% p(x|i)= Gp(x)p(i|x)\ge Gp(x)p(k|x)=p(x|k),
% \]
% hence $p(x|i)=\max_{j\in\mathcal{Y}}p(x|j)$ and
% \[
% \E[P_e] =1-\sum_{i=1}^G p(i)\int_{A_i}p(x|i)dx=1 -\frac{1}{G}\int
% \max_{i\in\mathcal{Y}} p(x|i)dx.
% \]

\section{Existence of Measures with Certain
Moments}\label{sec:tm}
The proof of our theorem reduces to the problem of how to check if a
given list of $n$
numbers could be the moments of some measure. This problem is called the truncated moment problem; here we review the relevant solutions by Curto and Fialkow \cite{Curto:1991vx}.

Suppose we are given a list of numbers $\gamma =
\{\gamma_0,\gamma_1,\ldots,\gamma_n\}$, with $\gamma_0 > 0$. Can this
collection be a list of moments for some positive Borel measure $\nu$
on $\R$ such that
\begin{equation}\label{eqn:moments}
\gamma_i=\int s^id\nu(s)?
\end{equation}
Let $k = \lfloor n/2\rfloor$, and construct a Hankel matrix $A(k)$
from $\gamma$ where the $i${\it th} row of $A$ is $[\gamma_{i-1} \:
\gamma_{i} \ldots \gamma_{i-1+k}]$. For example, for $n=2$ or $n=3$,
$k=1$:
\begin{equation*}%\label{eqn:Amatrixforn2}
A(1) = \begin{bmatrix}
 \gamma_0 &\gamma_1\\
 \gamma_1 &\gamma_2
\end{bmatrix}.
\end{equation*}
Let $\mathbf{v}_j$ be the transpose of the vector $(\gamma_{i+j})_{i=0}^k$. For $0\le j\le k$ this vector is the $j+1${\it th} column of $A(k)$. Define $\rank(\gamma) = k+1$ if
$A(k)$ is invertible, and otherwise $\rank(\gamma)$ is the smallest $r$ such that
$\mathbf{v}_r$ is a linear combination of
$\{\mathbf{v}_0,\ldots,\mathbf{v}_{r-1}\}$. % such that
%$\mathbf{v}_r=\sum_{i=0}^{r-1}\phi_i\mathbf{v}_i$ for some $\{\phi_i\}$.

%If there exists a $\mu$ that satisfies (\ref{eqn:moments}), then
%there exists a solution with
%atomic measure. To be able to say more we have
%to distinguish between two cases, depending on $n$.

Then whether there exists a $\nu$ that satisfies (\ref{eqn:moments})
depends on $n$ and $k$:

\begin{enumerate}[1)]
\item If $n = 2k+1$, then there exists such a
  solution $\nu$ if $A(k)$ is positive semidefinite and $\mathbf{v}_{k+1}$ is
  in the range of $A(k)$.

\item If $n = 2k$, then there exists such a
  solution $\nu$ if $A(k)$ is positive semidefinite and $\rank(\gamma)
  = \rank(A(k))$.
\end{enumerate}
Also, if there exists a $\nu$ that satisfies
(\ref{eqn:moments}), then there definitely exists a solution with
atomic measure.

%%%%%%%%%%%%%%%%%%%%%%%%%%%%%%%%%%%%%%%%%%%%%%%%%%%%%%%%%%%%%%%%%%%%%%%
\bibliographystyle{IEEEtran}
% IEEE editors may delete the raggedright, but otherwise it's a mess!
%\begin{raggedright}
\bibliography{BayesErrorBiblio}

%\end{raggedright}
\end{document}

%% file: mathheader.tex
\usepackage{amsmath}
\usepackage{amsthm}
\usepackage[pdftex]{epsfig}
\usepackage[pdftex]{color}
\usepackage{mathrsfs} 
\usepackage{amsfonts}
\usepackage{amssymb}
\usepackage{bbm}

\usepackage{enumerate}

\newcommand{\R}{{\mathbb{R}}}

\newcommand{\abs}[1]{\left|#1\right|}

\theoremstyle{plain}

\newtheorem{thm}{Theorem}[section]
\newtheorem{lem}{Lemma}[section]
\newtheorem{cor}{Corollary}[section]

\theoremstyle{definition}

\newtheoremstyle{exercise}{3pt}{3pt}{\slshape\small}{}
{\bfseries}{:}{.5em}{}

\theoremstyle{exercise}

\theoremstyle{remark}

\DeclareMathOperator{\E}{\mathbbm{E}}

\numberwithin{equation}{section}

%% file: maxBayes9.bbl
\begin{thebibliography}{10}
\providecommand{\url}[1]{#1}
\csname url@samestyle\endcsname
\providecommand{\newblock}{\relax}
\providecommand{\bibinfo}[2]{#2}
\providecommand{\BIBentrySTDinterwordspacing}{\spaceskip=0pt\relax}
\providecommand{\BIBentryALTinterwordstretchfactor}{4}
\providecommand{\BIBentryALTinterwordspacing}{\spaceskip=\fontdimen2\font plus
\BIBentryALTinterwordstretchfactor\fontdimen3\font minus
  \fontdimen4\font\relax}
\providecommand{\BIBforeignlanguage}[2]{{%
\expandafter\ifx\csname l@#1\endcsname\relax
\typeout{** WARNING: IEEEtran.bst: No hyphenation pattern has been}%
\typeout{** loaded for the language `#1'. Using the pattern for}%
\typeout{** the default language instead.}%
\else
\language=\csname l@#1\endcsname
\fi
#2}}
\providecommand{\BIBdecl}{\relax}
\BIBdecl

\bibitem{HTF}
T.~Hastie, R.~Tibshirani, and J.~Friedman, \emph{The Elements of Statistical
  Learning}.\hskip 1em plus 0.5em minus 0.4em\relax New York: Springer-Verlag,
  2001.

\bibitem{BDA}
S.~Srivastava, M.~R. Gupta, and B.~A. Frigyik, ``Bayesian quadratic
  discriminant analysis,'' \emph{J. Mach. Learn. Res.}, vol.~8, pp. 1277--1305,
  2007.

\bibitem{CoverThomas}
T.~Cover and J.~Thomas, \emph{Elements of Information Theory}.\hskip 1em plus
  0.5em minus 0.4em\relax New York: John Wiley and Sons, 1991.

\bibitem{Dowson:73}
D.~C. Dowson and A.~Wragg, ``Maximum entropy distributions having prescribed
  first and second moments,'' \emph{IEEE Trans. on Information Theory},
  vol.~19, no.~5, pp. 689--693, 1973.

\bibitem{Poor:80}
V.~Poor, ``Minimum distortion functional for one-dimensional quantisation,''
  \emph{Electronics Letters}, vol.~16, no.~1, pp. 23--25, 1980.

\bibitem{Moulin:05}
P.~Ishwar and P.~Moulin, ``On the existence and characterization of the maxent
  distribution under general moment inequality constraints,'' \emph{IEEE Trans.
  on Information Theory}, vol.~51, no.~9, pp. 3322--3333, 2005.

\bibitem{Tryphon:06}
T.~T. Georgiou, ``Relative entropy and the multivariable multidimensional
  moment problem,'' \emph{IEEE Trans. on Information Theory}, vol.~52, no.~3,
  pp. 1052--1066, 2006.

\bibitem{FriedlanderGupta}
M.~P. Friedlander and M.~R. Gupta, ``On minimizing distortion and relative
  entropy,'' \emph{IEEE Trans. on Information Theory}, vol.~52, no.~1, pp.
  238--245, 2006.

\bibitem{Dukkipati:06}
A.~Dukkipati, M.~N. Murty, and S.~Bhatnagar, ``Nonextensive triangle equality
  and other properties of {Tsallis} relative entropy minimization,''
  \emph{Physica A: Statistical Mechanics and Its Applications}, vol. 361,
  no.~1, pp. 124--138, 2006.

\bibitem{Lanckreit:02}
G.~R.~G. Lanckriet, L.~{El Ghaoui}, C.~Bhattacharyya, and M.~I. Jordan, ``A
  robust minimax approach to classification,'' \emph{Journal Machine Learning
  Research}, vol.~3, pp. 555--582, 2002.

\bibitem{Csiszar:01}
I.~Csiszar, G.~Tusnady, M.~Ispany, E.~Verdes, G.~Michaletzky, and T.~Rudas,
  ``Divergence minimization under prior inequality constraints,'' \emph{IEEE
  Symp. Information Theory}, 2001.

\bibitem{Csiszar:08}
I.~Csiszar and F.~Matus, ``On minimization of multivariate entropy
  functionals,'' \emph{IEEE Symp. Information Theory}, 2008.

\bibitem{Antos:99}
A.~Antos, L.~Devroye, and L.~Gy\"{o}rfi, ``Lower bounds for {Bayes} error
  estimation,'' \emph{IEEE Trans. Pattern Analysis and Machine Intelligence},
  vol.~21, no.~7, pp. 643--645, 1999.

\bibitem{Curto:1991vx}
R.~E. Curto and L.~A. Fialkow, ``Recursiveness, positivity, and truncated
  moment problems,'' \emph{Houston Journal of Mathematics}, vol.~17, no.~4, pp.
  603--635, 1991.

\bibitem{Bertsimas:2005fe}
D.~Bertsimas and I.~Popescu, ``{Optimal Inequalities in Probability Theory: A
  Convex Optimization Approach},'' \emph{SIAM Journal on Optimization},
  vol.~15, no.~3, pp. 780--804, 2005.

\bibitem{Marshall:60}
A.~W. Marshall and I.~Olkin, ``Multivariate {Chebyshev} inequalities,''
  \emph{The Annals of Mathematical Statistics}, vol.~31, no.~4, pp. 1001--1014,
  1960.

\end{thebibliography}
